\documentclass[11pt]{article}

\usepackage[T1]{fontenc}
\usepackage{amsmath,amssymb,bbold}
\usepackage{graphicx}
\usepackage{subcaption}
\usepackage{enumerate}
\usepackage{todonotes}
\usepackage{hyperref}
\usepackage{geometry}
\usepackage{authblk}
\usepackage{color}
\usepackage{amsthm} 

\newtheorem{theorem}{Theorem}[section]
\newtheorem{lemma}[theorem]{Lemma}
\newtheorem{proposition}[theorem]{Proposition}

\renewenvironment{proof}{\par\noindent{\bf Proof.\ }}{\hfill$\square$\par}

\newcommand{\bitem}{\begin{itemize}}
\newcommand{\eitem}{\end{itemize}}
\newcommand{\mc}[1]{\mathcal{#1}}

\newcommand{\N}{\mathbb{N}}
\newcommand{\R}{\mathbb{R}}

\newcommand{\bpm}{\begin{pmatrix}}
\newcommand{\epm}{\end{pmatrix}}
\newcommand{\bvm}{\begin{vmatrix}}
\newcommand{\evm}{\end{vmatrix}}
\newcommand{\bsm}{\left(\begin{smallmatrix}}
\newcommand{\esm}{\end{smallmatrix}\right)}
\newcommand{\T}{\top}

\newcommand{\ol}[1]{\overline{#1}}

\newcommand{\wt}[1]{\widetilde{#1}}
\newcommand{\la}{\langle}
\newcommand{\ra}{\rangle}

\newcommand{\mfk}[1]{\mathfrak{#1}}

\newcommand{\w}{\omega}

\newcommand{\gdw}{\Leftrightarrow}

\newcommand{\eins}{\mathbb{1}}

\DeclareMathSymbol{\mydiv}{\mathbin}{symbols}{"04}

\DeclareMathOperator{\Diag}{Diag}

\DeclareMathOperator{\argmax}{arg max}

\DeclareMathOperator{\vvec}{vec}

\DeclareMathOperator{\ggrad}{grad}

\DeclareMathOperator{\Var}{var}

\newcommand{\sst}[1]{{\scriptscriptstyle #1}}

\makeatletter
\def\widebreve{\mathpalette\wide@breve}
\def\wide@breve#1#2{\sbox\z@{$#1#2$}%
     \mathop{\vbox{\m@th\ialign{##\crcr
\kern0.08em\brevefill#1{0.8\wd\z@}\crcr\noalign{\nointerlineskip}%
                    $\hss#1#2\hss$\crcr}}}\limits}
\def\brevefill#1#2{$\m@th\sbox\tw@{$#1($}%
  \hss\resizebox{#2}{\wd\tw@}{\rotatebox[origin=c]{90}{\upshape(}}\hss$}
\makeatletter

\title{\textbf{Riemannian Patch Assignment Gradient Flows}}

\author[1]{Daniel Gonzalez-Alvarado}
\author[1]{Fabio Schlindwein}
\author[1,4]{Jonas Cassel}
\author[2]{Laura Steingruber}
\author[3]{Stefania Petra}
\author[1,4]{Christoph Schn\"orr}

\affil[1]{Inst. for Mathematics, Image and Pattern Analysis Group, Heidelberg University}
\affil[2]{Inst. of Theoretical Medicine, Anatomy and Cell Biology, University of Augsburg}
\affil[3]{Inst. of Mathematics, Math. Imaging Group \& CAAPS, University of Augsburg}
\affil[4]{Research Station Geometry and Dynamics, Heidelberg University}

\date{} 

\begin{document}

\maketitle

\begingroup
\renewcommand\thefootnote{}
\footnotetext{Accepted for publication in the Scale Space and Variational Methods in Computer Vision (SSVM) 2025 conference, to appear in the Lecture Notes in Computer Science (LNCS) series by Springer. The final version will be available on SpringerLink.}
\endgroup

\begin{abstract}
This paper introduces \textit{patch assignment flows} for metric data labeling on graphs. Labelings are determined by regularizing initial local labelings through the dynamic interaction of both labels and label assignments across the graph, entirely encoded by a dictionary of competing labeled patches and mediated by patch assignment variables. Maximal consistency of patch assignments is achieved by geometric numerical integration of a Riemannian ascent flow, as critical point of a Lagrangian action functional. Experiments illustrate properties of the approach, including uncertainty quantification of label assignments.

\end{abstract}

\section{Introduction}\label{sec:introduction}
\textit{Assignment flows} \cite{Astrom:2017} denote a class of dynamical systems evolving on statistical manifolds which serve as models for `neural ODEs' on graphs obtained by geometric flow integration. A basic instance \cite{Savarino:2021wt} are systems of the form $\dot W = R_{W}[\Omega W]$ (see Eq.~\ref{eq:assignment_flow}  for the general form) which are  Riemannian gradient flows with respect to a potential, parametrized by a \textit{spatial} interaction matrix $\Omega(t)$, comprising a non-local graph Laplacian and a term steering the flow towards discrete decisions (label assignment). The time-variant parameters $\Omega(t)$ can be learned from data. 

By vectorization (see Section \ref{sec:approach1} for details), this basic flow equation takes the form $\dot w = R^{\mfk{v}}_{w}[(\Omega\otimes I_{c}) w]$ which suggests to generalize the parametrization to $\dot w = R^{\mfk{v}}_{w}[(\Omega\otimes \Omega_{c}) w]$, in order to take  additionally into account the \textit{interaction} between $[c]=\{1,\dotsc,c\}$ \textit{labels}, to be jointly assigned to data observed on a graph via the assignment flow $w(t)$. In fact, this formulation is general enough to cover a broad range of multi-population and multi-game dynamics \cite{Boll:2024aa}.

The \textbf{purpose of this paper} is to encode both the \textit{spatial and label interaction} for regularized label assignment flows \textit{entirely} in terms of a \textit{dictionary} of \textit{labeled template patches}. To this end, we extend assignment flows to Riemannian \textit{patch assignment flows} using the very same mathematical framework, which regularize a given initial data labeling by maximizing the consistency of labeled patches over the graph. The final label assignment results from dynamically interacting labeled patches with  overlapping supports and from the underlying information geometry which enforces assignments.

\textbf{Related work.} There is a vast literature on image denoising and restoration using \textit{continuous} patches, either in a classical way by enforcing structure sparsity 
\cite{Chatterjee:2012aa,Peyre:2009aa,Yu:2012aa} or more recently by, e.g., Gaussian mixture priors
\cite{Aguerrebere:2017aa,Houdard:2018aa}. By contrast, our approach exploits \textit{labeled patches}, where each label represents an equivalence class of continuous signals, for the structured prediction of \textit{labelings} of metric data observed on a graph.

\textbf{Basic Notation.} 
We set $\eins_{n}=(1,1,\dotsc,1)^{\T}\in\R^{n}$ for any $n\in\N$. The canonical inner product of vectors or matrices is denoted by $\la\cdot,\cdot\ra$. For a scalar-valued smooth objective function defined on a Euclidean space or a Riemannian manifold, the Euclidean gradient is denoted by $\partial J$ and the Riemannian gradient by $\ggrad J$, respectively. For matrices $P\in\R^{n\times k}$, the row vectors are denoted by $P_{i},\; i\in[n]:=\{1,2,\dotsc,n\}$. The set $\mc{C}=[c]$ indexed classes (categories).

\textbf{Organization.}
Section \ref{sec:preliminaries} introduces notation and concepts required in the rest of the paper. Section \ref{sec:pafs} introduces the \textit{patch assignment flow}, whose properties are experimentally illustrated in Section \ref{sec:experiments}. We conclude in Section \ref{sec:conclusion}.

\section{Preliminaries}\label{sec:preliminaries}


\subsection{Graphs and Patch Dictionaries.}
In this paper, the domain of observed data and variables are connected oriented graphs $\mc G_{\mc V} = (\mc V, \mc E_{\mc V})$, with vertices $\mc V =[n]$ and edge set $\mc E_{\mc V} \subset \mc{V}\times\mc{V}$. The relation $ij\in \mc E_{\mc{V}}$ is also denoted as $i\to j$.
The adjacency matrix of $\mc G_{\mc V}$ is
\begin{align}\label{eq:def-A-GV}
    A_{\mc V;i,j} = \begin{cases}  1, & \text{if} \ ij \in \mc E_{\mc V}, \\ 0, & \text{otherwise.}\end{cases}
\end{align}
Since $\mc G_{\mc V}$ is oriented, $A_{\mc V}$ generally is asymmetric, i.e.~$ij\in\mc{E}_{\mc{V}}$ does \textit{not} generally imply that $ji\in\mc{E}_{\mc{V}}$. 
We say that $\mc G_{\mc V}$ is \textit{labeled} if there is a mapping  
\begin{equation}\label{eq:def-ell-map}
\ell_{\mc{V}}\colon \mc{V}\to\mc{C}
\end{equation}
that assigns to each vertex $i\in\mc{V}$ a class label $\ell_{\mc{V}}(i)\in\mc{C}$. 
For \textit{grid graphs} $\mc{G}_{\mc{V}}$, the \textit{support} of a patch centered at $i\in\mc{V}$ is denoted as $[i]_{\sst{\mc{V}}} \subset\mc{V}$, with patch size $|[i]_{\sst{\mc{V}}}|=:p,\,\forall i\in\mc{V}$. If $i\neq j$ and patches supported on $[i]_{\sst{\mc{V}}}$ and $[j]_{\sst{\mc{V}}}$ have the same size, then they can be mapped to each other by translation. Abstracting from the center point yields a \textit{labeled patch template},  denoted by $d$.

A \textit{labeled patch dictionary} $\mc{D}$ is a collection of labeled patch templates $\{d \colon d\in\mc{D}\}$, all having the same size $|d|=p\in\N$. Centering  $d$ at pixel $i\in\mc{V}$ is denoted as $d_{[i]}$\footnote{The consistent notation would be $d_{[i]_{\sst{\mc{V}}}}$. Yet we avoid the double subscript to simplify notation, and the context disambiguates the meaning $[i]_{\sst{\mc{V}}}\subset\mc{V}$ of the subscript $d_{[i]}$, rather than $[i]=\{1,\dotsc,i\}$.}. Given a graph $\mc G_{\mc V} = (\mc V, \mc E_{\mc V})$, each patch dictionary $\mc D$ induces a directed \textit{patch dictionary graph} 
\begin{equation}\label{eq:def-G-mcD}
\mc G_{\mc D} = (\mc D, \mc{E}_{\mc D}), 
\end{equation}
with patch templates $d, d'\in\mc{D}$ being adjacent if 
\begin{equation}
dd'\in\mc{E}_{\mc{D}}
\quad\Leftrightarrow\quad
(ij\in\mc{E}_{\mc{V}})\wedge (\w_{d_{[i]}d'_{[j]}}> 0),
\end{equation}
where $\w\colon\mc{D}\times\mc{D}\to\R_{\geq 0}$ is a nonnegative similarity function. A basic example is the normalized agreement of the patch templates $d_{[i]}, d'_{[j]}$ on the intersection of the supports $[i]_{\sst{\mc{V}}}\cap[j]_{\sst{\mc{V}}}$, 
\begin{equation}\label{eq:om-patch-weight}
\w_{d_{[i]}d'_{[j]}} = \frac{1}{p}\big|\{k\in [i]_{\sst{\mc{V}}}\cap[j]_{\sst{\mc{V}}}\colon d_{[i]}(k)=d'_{[j]}(k)\}\big|,\qquad p = |[i]_{\sst{\mc{V}}}| = |[j]_{\sst{\mc{V}}}|.
\end{equation}
Here, $d_{[j]}(k)$ denotes the value at vertex $k\in\mc{V}$ of patch $d$, centered at $j\in\mc{V}$. A special case is the \textit{binary} similarity function
\begin{equation}\label{eq:om-patch-binary}
\w_{d_{[i]}d'_{[j]}} = \begin{cases}
1, &\text{if}\;d|_{[i]_{\sst{\mc{V}}}\cap[j]_{\sst{\mc{V}}}}\equiv d'|_{[i]_{\sst{\mc{V}}}\cap[j]_{\sst{\mc{V}}}}, \\
0, &\text{otherwise.}
\end{cases}
\end{equation}
These similarity values only depend on the position of vertices $i$ and $j$ \textit{relative} to each other, but are translation invariant otherwise. 
As a consequence, for the specific case of a two-dimensional grid graph $\mc{G}_{\mc{V}}$ with edge set 
\begin{equation}\label{eq:edgeset-v-h}
\mc{E}_{\mc{V}} = \mc{E}^{h}_{\mc{V}} \dot\cup \mc{E}^{v}_{\mc{V}}
\end{equation}
comprising horizontal and vertical edges, 
any similarity function $\w$ defines corresponding asymmetric \textit{weighted patch template adjacency matrices}
\begin{subequations}\label{eq:def-Om-G-D}
\begin{align}
\Omega^{h}_{\mc{D}}\in\R_{\geq 0}^{|\mc{D}|\times|\mc{D}|},\qquad
\Omega^{h}_{\mc{D};d,d'}
&= \w_{d_{[i]}d'_{[j]}}\quad\text{for any}\; ij\in\mc{E}^{h}_{\mc{V}}
\\
\Omega^{v}_{\mc{D}}\in\R_{\geq 0}^{|\mc{D}|\times|\mc{D}|},\qquad
\Omega^{v}_{\mc{D};d,d'}
&= \w_{d_{[i]}d'_{[j]}}\quad\text{for any}\; ij\in\mc{E}^{v}_{\mc{V}}
\end{align}
\end{subequations}
of the underlying patch dictionary graph $\mc{G}_{\mc{D}}$. See Figure \ref{fig:patch-adjacency} for an illustration.

\begin{figure}[ht]
\centerline{
\includegraphics[width=0.25\textwidth]{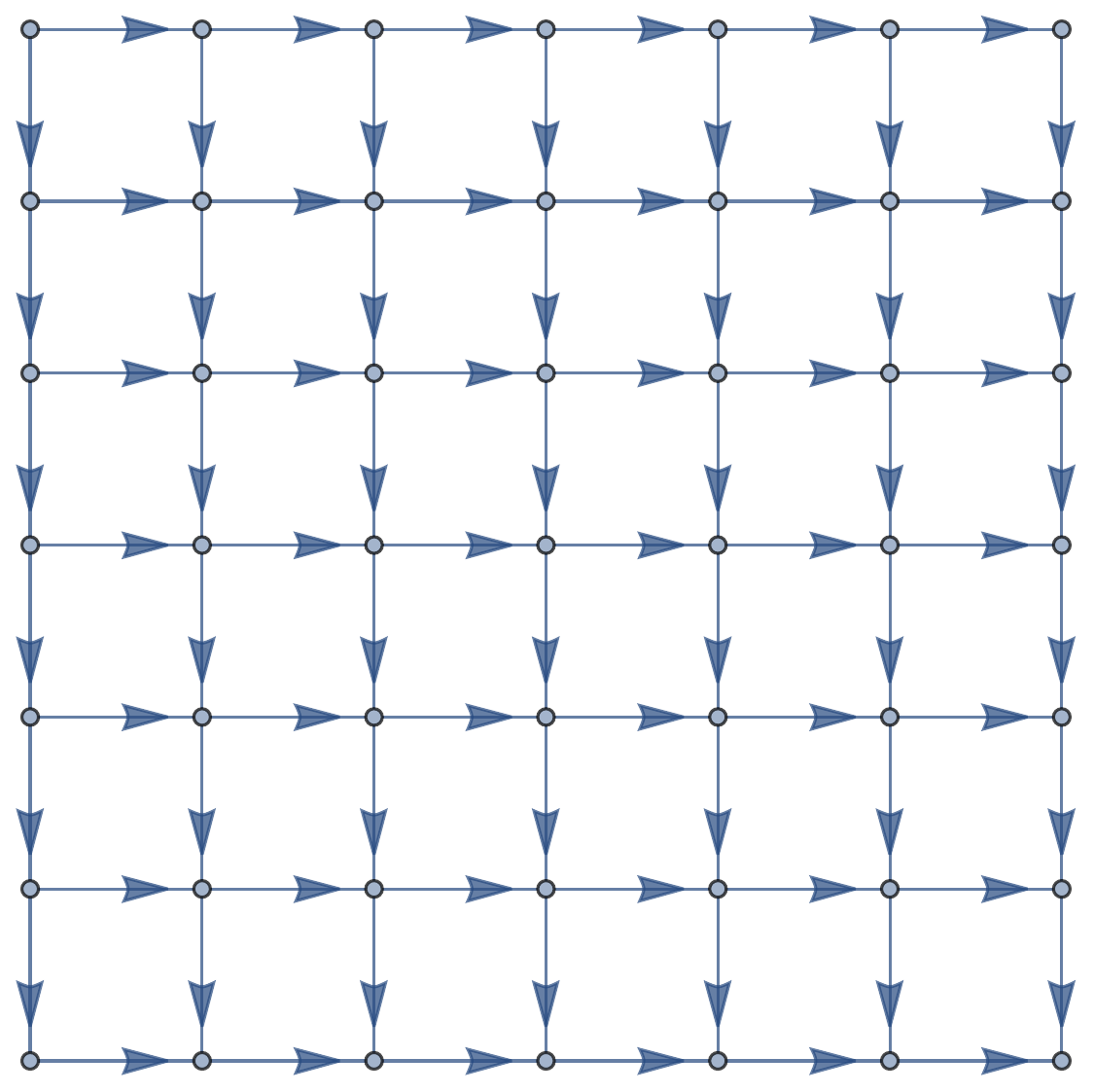}
\hspace{0.05\textwidth}
\includegraphics[width=0.25\textwidth]{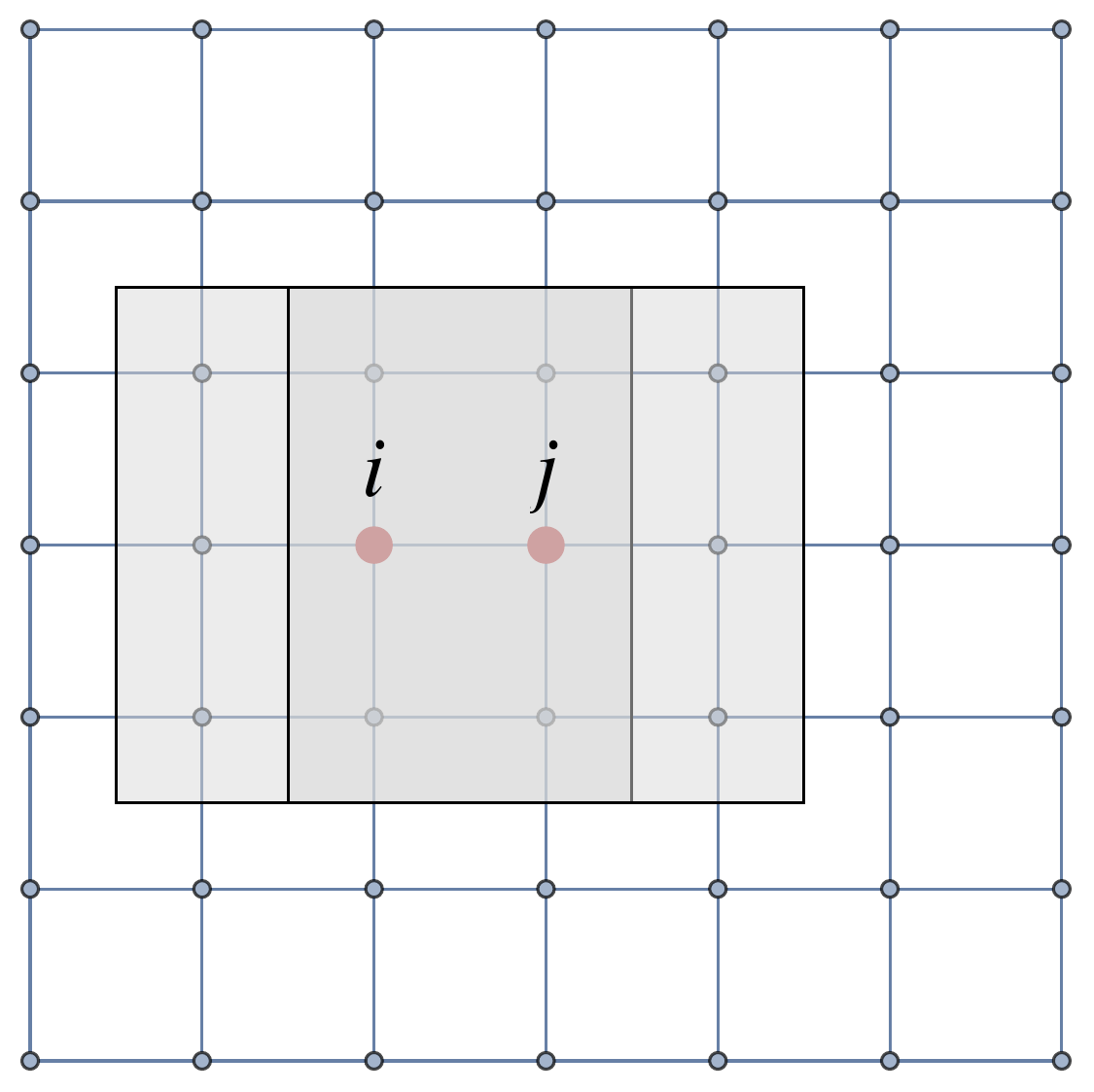}
\hspace{0.05\textwidth}
\includegraphics[width=0.28\textwidth]{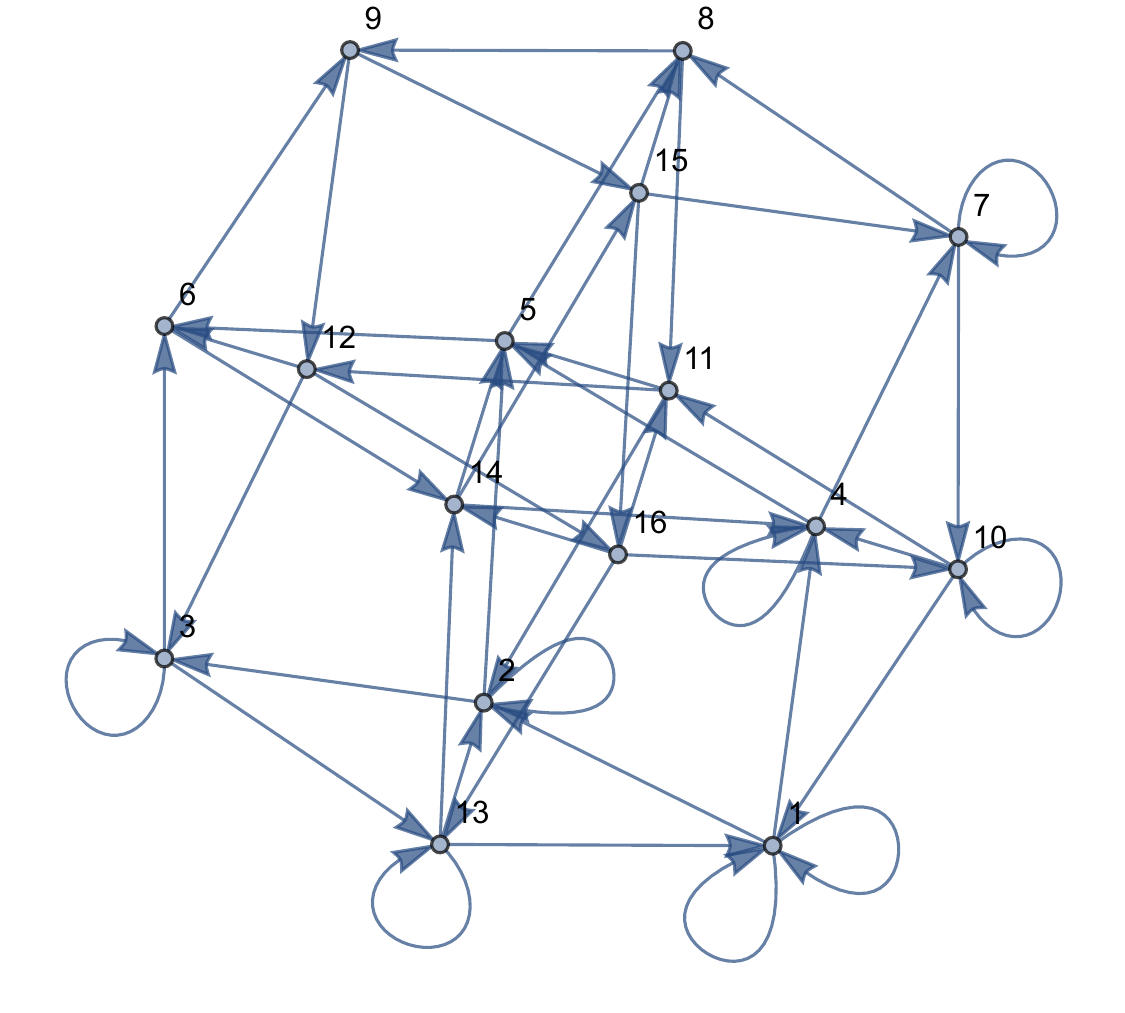}
}
\centerline{
\parbox{0.25\textwidth}{\centering (a)}
\hspace{0.05\textwidth}
\parbox{0.25\textwidth}{\centering (b)}
\hspace{0.05\textwidth}
\parbox{0.28\textwidth}{\centering (e)}
}
\vspace{0.2cm} 

\centerline{
\includegraphics[width=0.9\textwidth]{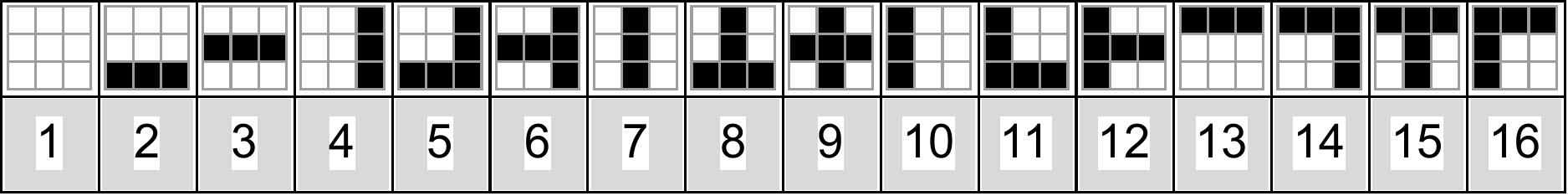}
}
\centerline{
\parbox{0.9\textwidth}{\centering (c)}
}
\vspace{0.2cm} 
\centerline{
\includegraphics[width=0.9\textwidth]{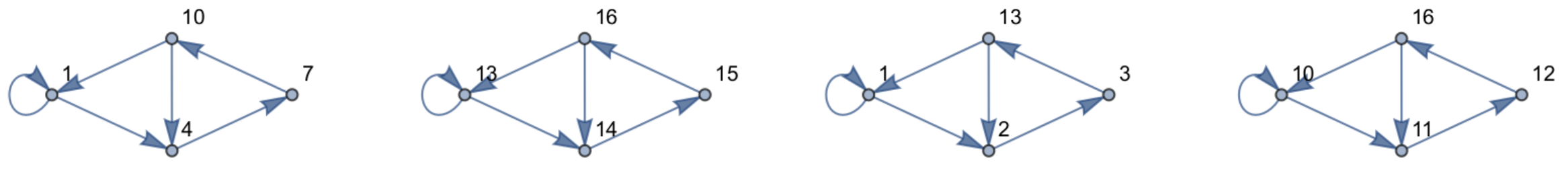}
}
\centerline{
\parbox{0.9\textwidth}{\centering (d)}
}

\caption{
\textbf{(a)} Directed grid graph $\mathcal{G}_{\mathcal{V}}$ with an arbitrary orientation. \textbf{(b)} Intersection supports $[i]_{\mc{V}}$ and $[j]_{\mc{V}}$ of patches centered at an horizontal edge. \textbf{(c)} A small patch dictionary $\mathcal{D}$ with 16 binary labeled patch templates. \textbf{(d)} Few patch template adjacency relations based on \eqref{eq:om-patch-binary}. The first two examples show horizontal adjacency relations and the last two vertical ones. For example, regarding the first graph, $7\to 10$ means moving from patch template 7 one vertex to the right (or equivalently, shifting patch template 7 one vertex to the left) may result in patch template 10, whereas shifting patch template 10 to the left may result in either patch template 1 or patch template 4. A loop at vertex $d$ indicates that patch template $d$ is \textit{self-adjacent}. \textbf{(e)} The entire resulting patch dictionary graph $\mathcal{G}_{\mathcal{D}}$ \eqref{eq:def-G-mcD} comprising both horizontal and vertical patch template adjacencies.}

\label{fig:patch-adjacency}
\end{figure}


\subsection{Assignment Manifold, Assignment Flows.}
We introduce few definitions from information geometry required in connection with assignment flows \cite{Amari:2000}, \cite{Astrom:2017}, \cite[Chapter 2]{Ay:2017}. Given a set $\mc C$ of category labels, we denote by $\mc S_{c}=\{w \in \R_{>0}: \langle \eins_{c}, w \rangle = 1\}$ the relative interior of the probability simplex over $c = |\mc{C}|$ labels. This space forms a smooth manifold with trivial tangent bundle $T \mc S_c = \mc S_c \times T_0$ and tangent space at any $w\in\mc{S}_{c}$ given by $T_0 := T_w \mc S_c = \{v \in \R^{c}: \langle \eins_{c} , v \rangle = 0 \} $. $(\mc{S}_{c},g^{\mc S_c})$ becomes a Riemannian manifold when equipped with the Fisher-Rao metric
\begin{align}\label{eq:def-FR-metric}
    g^{\mathcal{S}_c}_w(u,v) = \sum_{\ell\in[c]} \frac{u_\ell v_\ell}{w_\ell} = \langle u, \operatorname{Diag}(w)^{-1}v \rangle,\qquad w\in \mc S_c, \ u,v \in T_0.
\end{align}
The inverse metric tensor on $\mc S_c$, expressed in ambient coordinates and referred to as the \textit{replicator operator}, is denoted as $R_w$ and has the form
\begin{align}\label{eq:replicator-mat}
    R_w: \R^c \to T_0, \qquad x\mapsto  w\circ x-\langle x,w\rangle w,
\end{align}
where $w\circ x:=\Diag(w) x$ denotes the entrywise Hadamard (or Schur) product.

Given a graph $\mc G_{\mc V} = ( \mc V, \mc E_{\mc V}) $, with $\mc V =[n]$, 
the \textit{assignment manifold} is the product manifold
\begin{align}\label{eq:assignment_mfd} 
    \mc{S}_{c}^{n} := \mathcal{S}_c\times \dots \times \mathcal{S}_c \qquad (n=|\mc{V}|\;\text{factors}),
\end{align}
which becomes a Riemannian manifold when equipped with the natural extension of \eqref{eq:def-FR-metric} to the Fisher-Rao \textit{product} metric. 
Each point $W = (W_{1},\dotsc,W_{n})^{\T}$ on this manifold represents discrete probability vectors $W_{i}\in\mc{S}_{c},\; i\in\mc{V}$, called \textit{assignment vectors}, as variables to be determined for possible assignments of labels $\mc{C}=[c]$ to vertex $i$.

Assignment flows \cite{Astrom:2017} constitute an approach for determining a labeling function \eqref{eq:def-ell-map}, based on the evolution of label assignment vectors as integral curves of vector fields on the assignment manifold, governed by a coupled system of replicator equations 
\begin{subequations}\label{eq:assignment_flow}
\begin{align} 
    \dot W :=  (\dot W_1, \dots , \dot W_n)^{\T} 
    &= (R_{W_1}[F(W)_{1}], \dots ,  R_{W_n}[F(W)_{n}])^{\T} 
    \\
    &=: R_{W}[F(W)].
\end{align}
\end{subequations}
The function $F:\mc{S}_{c}^{n} \to \R^{n\times c}$ parametrizing the flow is often referred to as \textit{fitness} (or \textit{payoff}) \textit{function} in evolutionary game theory and population dynamics, where analyses of individual (uncoupled) replicator equations have provided foundational insights into evolutionary stability and strategic dynamical games \cite{Hofbauer:1998,Sandholm:2010}.
For a more detailed exposition of assignment flows, we refer to \cite{Schnorr:2020}.

For a smooth objective function $J:\mc{S}_{c}^{n} \to \R$, the \textit{Riemannian gradient flow} emanating from some initial point $W_0 \in \mc{S}_{c}^{n}$ is given by
\begin{align}\label{eq:AF-J}
    \dot W = \ggrad(J(W)) = R_{W}[\partial J(W)], \quad W(0) = W_0,
\end{align}
which is a particular case of the general assignment flow equation \eqref{eq:assignment_flow}.

Geometric integration of such flows can be performed numerically in a stable way using any of the methods worked out by \cite{Zeilmann:2020}. Under suitable conditions on the fitness function \cite{Zern:2022}, these schemes guarantee convergence $\lim_{t\to\infty} W_{i}(t) =: e_{\ell_{\mc{V}}(i)},\; i\in\mc{V}$ to some labeling function \eqref{eq:def-ell-map}, and they are amenable to learning parameters of the affinity function $F(W)$ from data \cite{Huhnerbein:2021,Zeilmann:2023,Boll:2024,Cassel:2024}.

\section{Patch Assignment Flows (P-AFs)}\label{sec:pafs}

This section introduces \textit{patch assignment flows}  (short: P-AFs) (Section \ref{sec:approach1}). We verify that P-AFs do not depend on the orientation of the underlying graph (Section \ref{sec:approach2}) and show that any P-AF can be characterized as stationary point of a specific action functional (Section \ref{sec:approach3}).

\vspace{-0.25cm}
\subsection{Flow Definition}\label{sec:approach1}

Consider the scenario described in Section \ref{sec:preliminaries} with an underlying graph $\mc G_{\mc V} = ( \mc V, \mc E_{\mc V}) $ and a labeled patch dictionary $\mc{D}$ with corresponding patch dictionary graph $\mc G_{\mc D} = (\mc D, \mc E_{\mc{D}})$. 

Similar to the definition of the assignment manifold in \eqref{eq:assignment_mfd}, we define the \textit{patch assignment manifold} as the product manifold
\begin{align}
    \mc{S}^{n}_{|\mc D|} := \mc S_{|\mc D|}  \times \cdots \times \mc S_{|\mc D|}  \quad (n=|\mc{V}| \;\text{factors}),
\end{align}
where $\mc S_{|\mc D|} := \{p \in \R_{>0}^{|\mc D|}: \langle \eins_{|\mc D|}, p \rangle =1 \}$ denotes the relative interior of the probability simplex with discrete measures over $|\mc D|$ categories. The pair $( \mc{S}^{n}_{|\mc D|}, g^{ \mc{S}^{n}_{|\mc D|}})$ is a Riemannian manifold with Fisher-Rao product metric $g^{ \mc{S}^{n}_{|\mc D|}}$. 

The \textbf{objective} is to define a Riemannian gradient flow equation in terms of a \textit{patch assignment vector field} $P=(P_{1},\dotsc,P_{n})^{\T}\in\mc{S}^{n}_{|\mc{D}|}$, analogous to \eqref{eq:AF-J}, whose limit $\lim_{t\to\infty} P(t)$ encodes a labeling function $\ell_{\mc{V}}\colon\mc{V}\to\mc{C}$ that assigns to each vertex $i\in\mc{V}$ \textit{both}
\begin{enumerate}[(i)]
\item a labeled patch template $d_{[i]}$ centered at vertex $i\in\mc{V}$, determined as $d=\argmax_{d\in|\mc{D}|}P_{i d}(\infty)\in\mc{D}$ from the dictionary $\mc{D}$,  and
\item a class label $d_{[i]}(k)\in\mc{C}$ corresponding to the center value at the position $k\in[i]_{\mc{V}}$ of the assigned patch template $d_{[i]}$.
\end{enumerate}
Thus, the P-AF to be devised is supposed to determine a labeling of all vertices, just as the basic assignment flow \eqref{eq:AF-J}, yet with a novel way for regularizing label assignments that is entirely induced by the interaction of \textit{labeled} patch templates with overlapping supports across the graph. As a consequence, regularization is \textit{completely} represented by the labeled patch dictionary $\mc{D}$ and the patch dictionary graph $\mc{G}_{\mc{D}}$ \eqref{eq:def-G-mcD}.

Recall the definitions of the adjaceny matrices \eqref{eq:def-A-GV} and \eqref{eq:def-Om-G-D} corresponding to $\mc{G}_{\mc{V}}$ and $\mc{G}_{\mc{D}}$. 
We consider the optimization problem
\begin{equation}\label{eq:def-J-P-AF}
\sup_{P\in\mc{S}^{n}_{|\mc{D}|}} J(P),\qquad
J(P) := \big\la P, A^{h}_{\mc{V}} P (\Omega^{h}_{\mc{D}})^{\T} + A^{v}_{\mc{V}} P (\Omega^{v}_{\mc{D}})^{\T}\big\ra,
\end{equation}
where $A^{h}_{\mc{V}}+A^{v}_{\mc{V}}=A_{\mc{V}}$, due to \eqref{eq:def-A-GV} and \eqref{eq:edgeset-v-h}.
\begin{lemma}[\textbf{maximizing patch consistency}]
Solving \eqref{eq:def-J-P-AF} is equivalent to maximizing
\begin{equation}\label{eq:J-P-explicit}
J(P) = \sum_{ij\in\mc{E}^{h}_{\mc{V}}}\la P_{i}, \Omega^{h}_{\mc{D}} P_{j}\ra + \sum_{ij\in\mc{E}^{v}_{\mc{V}}}\la P_{i}, \Omega^{v}_{\mc{D}} P_{j}\ra.
\end{equation}
\end{lemma}
\begin{proof}
\eqref{eq:J-P-explicit} results from \eqref{eq:def-J-P-AF} by taking into account the adjacency relation of $\mc{G}_{\mc{V}}$ given by $A^{h}_{\mc{V}}+A^{v}_{\mc{V}}=A_{\mc{V}}$, due to \eqref{eq:def-A-GV} and \eqref{eq:edgeset-v-h}. The inner product with the first term on the right hand side of \eqref{eq:def-J-P-AF} reads
\begin{equation}
\sum_{ij\in\mc{E}^{h}_{\mc{V}}} \la P_{i},(P (\Omega^{h}_{\mc{D}})^{\T})_{j}^{\T}\ra
=\sum_{ij\in\mc{E}^{h}_{\mc{V}}} \la P_{i},\Omega^{h}_{\mc{D}} P_{j}\ra,
\end{equation} 
and likewise for the second term. $\quad\Box$
\end{proof}
Using the row-stacking vectorization map $\vvec_r(ABC) = (A\otimes C^\T) \vvec_r(B)$ \cite{Loan:2000}, we define the vectorized objective function 
\begin{equation}
J^{\mfk{v}}(p^\mfk{v}) := J(\vvec_r(P)),\qquad
p^{\mfk{v}} := \vvec_{r}(P) 
= (P_{1}^{\T},\dotsc,P_{n}^{\T})^{\T} 
\end{equation}
and obtain
\begin{subequations}
\begin{align}
\partial J^{\mfk{v}}(p^\mfk{v})
&= \big(A^{h}_{\mc{V}}\otimes \Omega^{h}_{\mc{D}} + (A^{h}_{\mc{V}} \otimes \Omega^{h}_{\mc{D}})^{\T}
\\ &\qquad\qquad
+ A^{v}_{\mc{V}}\otimes \Omega^{v}_{\mc{D}} + (A^{v}_{\mc{V}} \otimes \Omega^{v}_{\mc{D}})^{\T}
\big) p^\mfk{v}
\\
\gdw\qquad 
\partial J(P) &= \mc{A}_{\Omega}(P) + \mc{A}^{\T}_{\Omega}(P),
\\ \label{eq:A-Om-P}
&\quad 
\mc{A}_{\Omega}(P) 
:= A^{h}_{\mc{V}} P (\Omega^{h}_{\mc{D}})^{\T}
+ A^{v}_{\mc{V}} P (\Omega^{v}_{\mc{D}})^{\T},
\\ \label{eq:A-Om-T-P}
&\quad 
\mc{A}^{\T}_{\Omega}(P) 
:= (A^{h}_{\mc{V}})^{\T} P (\Omega^{h}_{\mc{D}})
+ (A^{v}_{\mc{V}})^{\T} P (\Omega^{v}_{\mc{D}}),
\end{align}
\end{subequations}
and analogous to \eqref{eq:AF-J} the \textbf{patch assignment flow (P-AF) equation}
\begin{subequations}\label{eq:def-P-AF}
\begin{align}
\dot P(t) 
&= \ggrad\big(J(P)\big) 
= R_{P(t)}[\partial J(P(t))]
\\
&= R_{P(t)}[\mc{A}_{\Omega}(P) + \mc{A}^{\T}_{\Omega}(P)],\qquad P(0)=:P_{0}\in\mc{S}^{n}_{|\mc{D}|},
\end{align}
\end{subequations}
where $R_{P(t)}[\cdot]$ acts row-wise as in \eqref{eq:assignment_flow}. A canonical choice of the initial point $P_{0}$ is described in Section \ref{sec:Implementation-Details}.

\vspace{-0.3cm}
\subsection{Independence of Graph Orientation}\label{sec:approach2}

We show that the orientation of the graph $\mc{G}_{\mc{V}}$ (cf.~Figure \ref{fig:patch-adjacency}(a)) can be chosen arbitrarily.
\begin{proposition}
The patch assignment flow solving \eqref{eq:def-P-AF} does not depend on the orientation of the underlying graph $\mc{G}_{\mc{V}}$.
\end{proposition}
\begin{proof}[sketch]
Let $\wt{\mc G}_{\mc V} = (\mc V , \wt{\mc E}_{\mc V})$ denote the transpose of $\mc G_{\mc V}$, that is the graph with the same vertex set $\mc{V}$ but with reversed edge orientations: 
    $ij\in \mc {E}_{\mc V} \;\Leftrightarrow\; ji\in \wt {\mc E}_{\mc V},\;\text{for all } i,j\in \mc V$. 
Since the orientation affects the patch template adjacency matrix by \eqref{eq:def-Om-G-D}, the resulting patch adjacency graph $\wt{\mc{G}}_{\mc{D}}$ becomes transposed as well. As a result, following Section \ref{sec:approach1}, the corresponding objective then reads $\wt{J}(P) = \la P, \mc{A}^{\T}_{\Omega}(P)\ra$ and one easily verifies
$\ggrad\wt{J}(P) = \ggrad J(P)$.

Due to the disjoint decomposition \eqref{eq:edgeset-v-h} and the corresponding decomposition of the objective function \eqref{eq:def-J-P-AF}, this equation also holds if either the horizontal or the vertical edge orientations only are reversed. $\quad\Box$
\end{proof}
\begin{figure}[t]
\centerline{
\includegraphics[height=0.2\textwidth]{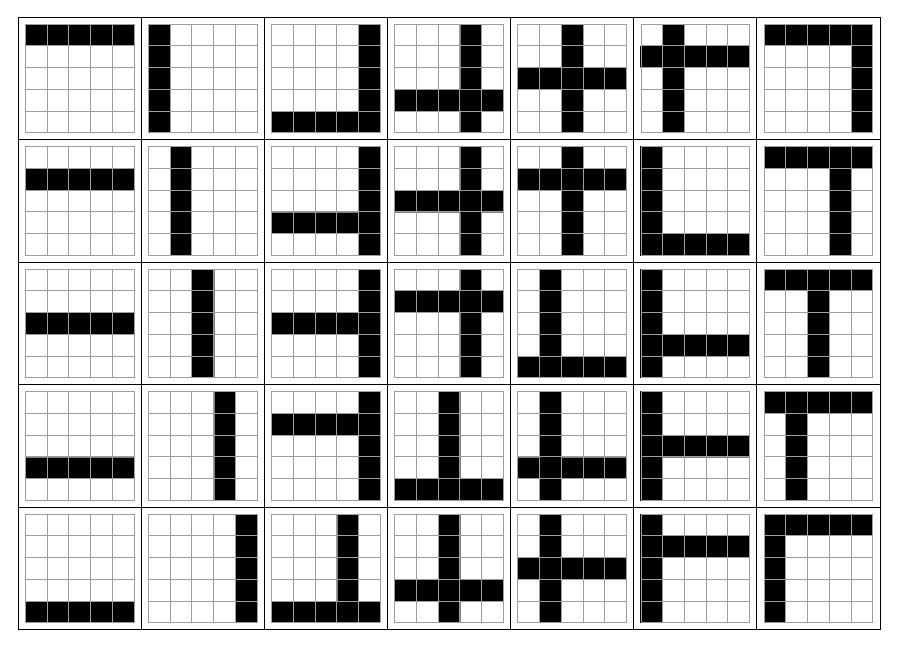}\hfill
\includegraphics[height=0.2\textwidth]{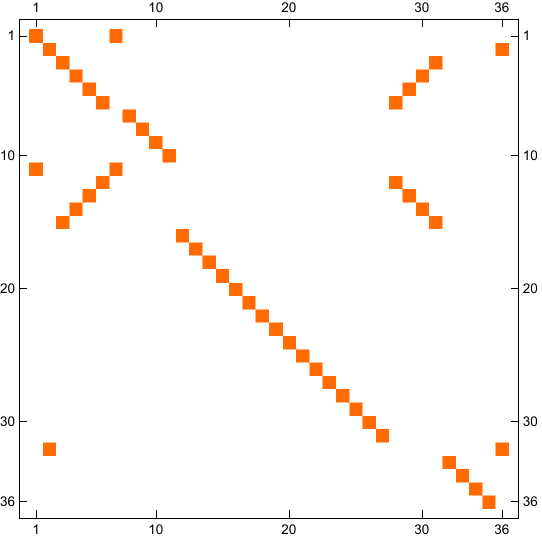}\hfill
\includegraphics[height=0.2\textwidth]{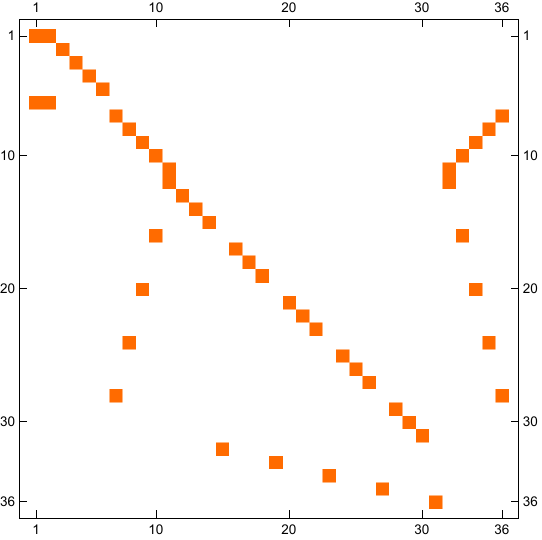}\hfill
\includegraphics[height=0.25\textwidth]{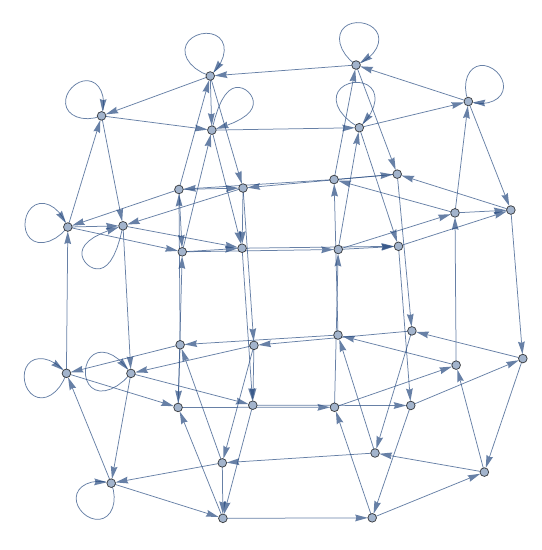}
}
\centerline{
\parbox{0.2\textwidth}{\centering (a)}\hfill
\parbox{0.2\textwidth}{\centering (b)}\hfill
\parbox{0.2\textwidth}{\centering (c)}\hfill
\parbox{0.2\textwidth}{\centering (d)}
}
\caption{
\textbf{(a)} A small dictionary of $5\times 5$ binary template patches, complemented by the constant background patch, which slightly extends the scenario of Figure \ref{fig:patch-adjacency}, yet still models a `small world of binary crossing structure'. \textbf{(b), (c)} Patch template adjacency matrices \eqref{eq:def-Om-G-D} for horizontal and vertical edges, respectively. \textbf{(d)} The resulting dictionary dictionary graph is highly symmetric. The 10 line patches (two leftmost columns in (a)) and the constant patch (not shown in (a)) are self-adjacent and correspond to loops in (d).
}
\label{fig:dictionary-lines-5x5}
\end{figure}

\subsection{Action Functional}\label{sec:approach3}

In \cite{Savarino:2024aa} it has been shown that the assignment flow determined by \eqref{eq:assignment_flow} can be characterized as stationary point of a Lagrangian action functional, provided the affinity function $F$ satisfies a corresponding assumption. It turns out that this result applies to the patch assignment flow \eqref{eq:def-P-AF} which essentially has the same mathematical structure.
\begin{proposition}[\textbf{P-AF: action functional, stationary point}]
The solution $P(t)$ of the patch assignment flow equation \eqref{eq:def-P-AF} is a critical point of the action functional
\begin{equation}
\mc{L}(P) = \frac{1}{2}\int_{t_{0}}^{t_{1}}\Big(\|\dot P(t)\|_{g}^{2} + \sum_{i\in\mc{V}}\Var_{P_{i}(t)}\big[\big(\mc{A}_{\Omega}(P) + \mc{A}^{\T}_{\Omega}(P)\big)_{i}\big] dt,
\end{equation}
where $\|\cdot\|_{g}$ denotes the Fisher-Rao product metric of the patch assignment manifold $\mc{S}^{n}_{|\mc{D}|}$.
\end{proposition}
\begin{proof}
Inspecting \cite[Theorem 3.3]{Savarino:2024aa} shows that a sufficient condition is that the differential $dF(W)$ of the affinity function of \eqref{eq:assignment_flow} is self-adjoint. In the present case, the affinity function reads $F(P) = \mc{A}_{\Omega}(P) + \mc{A}^{\T}_{\Omega}(P)$. The explicit form of the right-hand side given by \eqref{eq:A-Om-P} and \eqref{eq:A-Om-T-P} shows that it is linear in $P$ and that the differential is a symmetric matrix. $\quad\Box$ 
\end{proof}

Critical points of an action functional satisfy the corresponding  Euler-Lagrange equation, which define the equations of motion for a particle in a classical system described by the action. This gives an additional and insightful perspective on the solutions of the P-AF \eqref{eq:def-P-AF}.

\section{Experiments and Discussion}\label{sec:experiments}

%
\begin{figure}[t]
\centerline{
\includegraphics[height=0.15\textwidth]{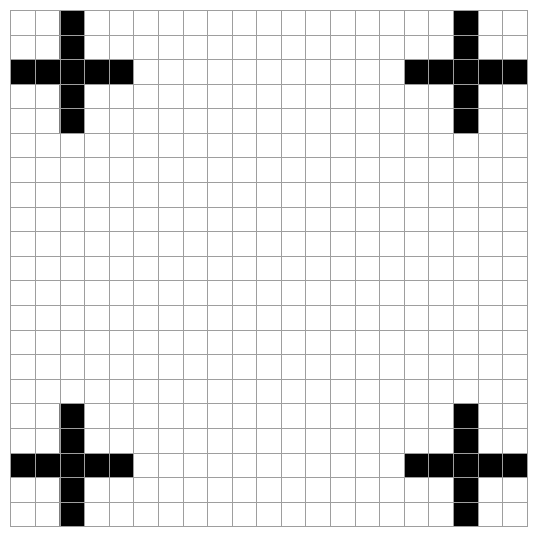}\hspace{0.05\textwidth}
\includegraphics[height=0.15\textwidth]{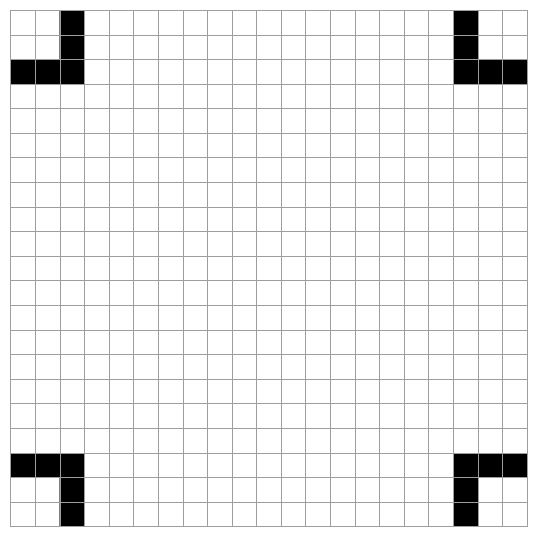}\hspace{0.05\textwidth}
\includegraphics[height=0.15\textwidth]{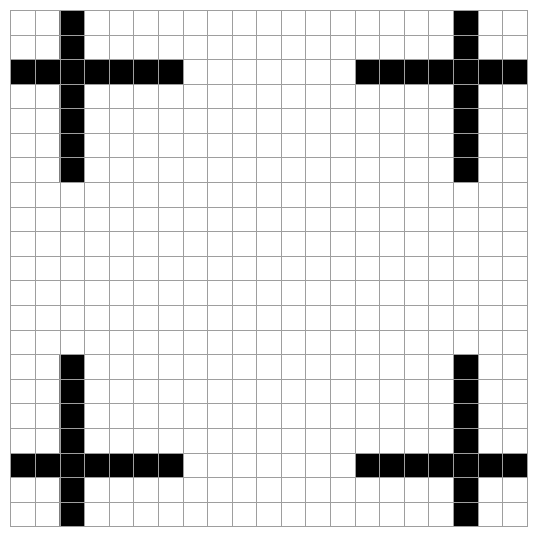}\hspace{0.05\textwidth}
\includegraphics[height=0.15\textwidth]{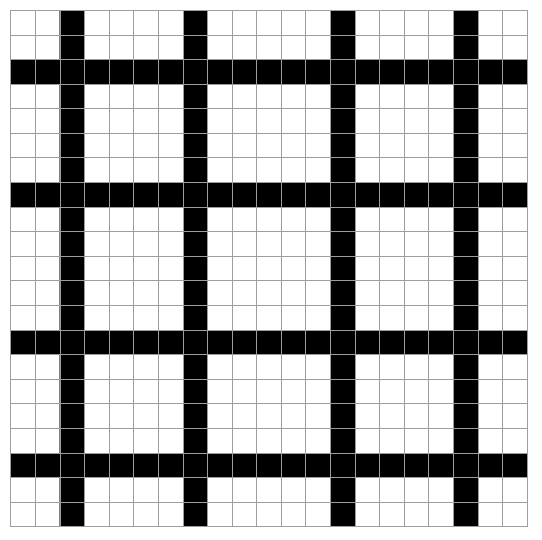}
}
\centerline{
\parbox{0.15\textwidth}{\centering (a)}\hspace{0.05\textwidth}
\parbox{0.15\textwidth}{\centering (b)}\hspace{0.05\textwidth}
\parbox{0.15\textwidth}{\centering (c)}\hspace{0.05\textwidth}
\parbox{0.15\textwidth}{\centering (d)}
}
\vspace{0.05cm}
\centerline{
\hspace{0.22\textwidth}
\includegraphics[height=0.15\textwidth]{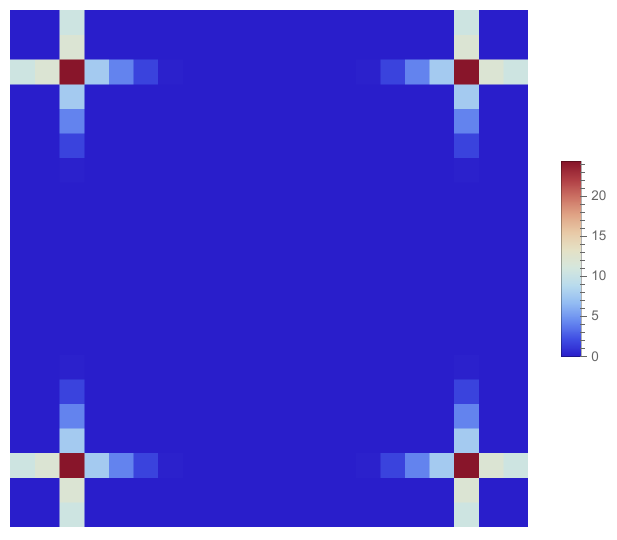}\hspace{0.03\textwidth}
\includegraphics[height=0.15\textwidth]{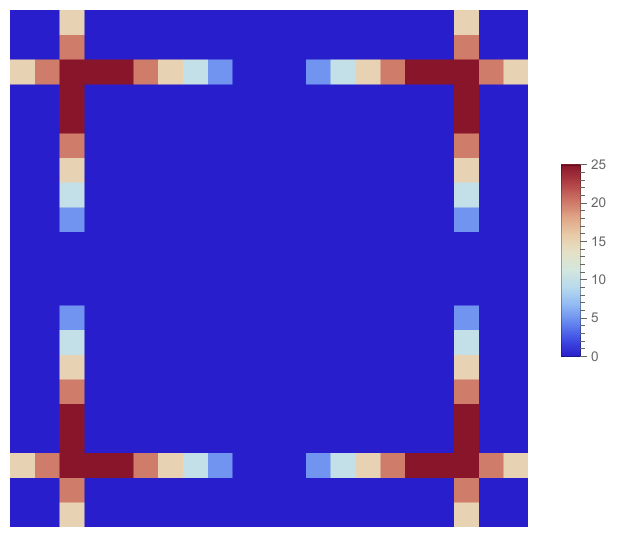}\hspace{0.03\textwidth}
\includegraphics[height=0.15\textwidth]{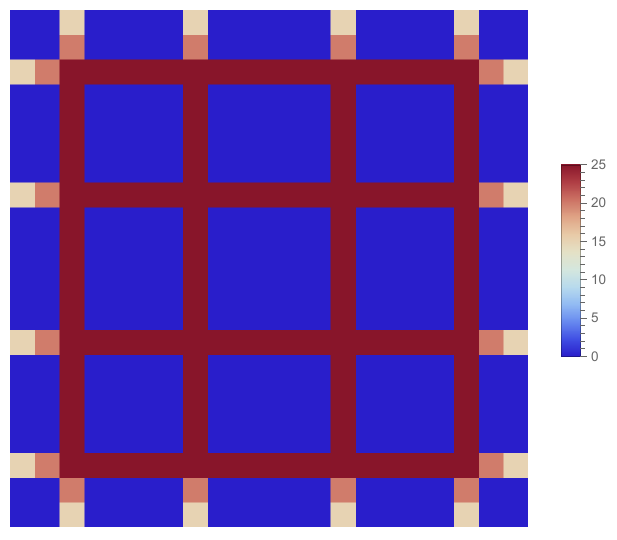}
}
\centerline{
\hspace{0.2\textwidth}
\parbox{0.15\textwidth}{\centering (e)}\hspace{0.05\textwidth}
\parbox{0.15\textwidth}{\centering (f)}\hspace{0.05\textwidth}
\parbox{0.15\textwidth}{\centering (g)}
}
\caption{
\textbf{(a)} Input data, to be regularized by the patch-AF using the dictionary of Fig.~\ref{fig:dictionary-lines-5x5}. \textbf{(b)-(d)} Labelings returned by the patch-AF when background and foreground labels of dictionary patches are equally important (ratio $1.0/1.0$; (b)) or not (ratio $1.0/1.2$ (c); ratio $1.0/1.5$ (d)). \textbf{(e)--(g)} Uncertainty quantification of the labelings above, due to the \textit{symmetry} of the patch dictionary graph and the corresponding \textit{multiplicity} of \textit{locally consistent} labelings (see text). The colors `red' and `blue' signal \textit{unique} fore- and background labelings, respectively, whereas `white' signals \textit{uncertainty} and plausible alternative labelings. This result illustrates that patch assignment flows enable both labeling pattern \textit{suppression} and labeling pattern \textit{formation}.
}
\label{fig:5x5-lines}
\end{figure}

This section provides experimental results in order to validate the novel patch assignment flow (P-AF) approach.
\vspace{-0.15cm}
\begin{enumerate}[(1)]
\item Section \ref{sec:Implementation-Details} specifies implementation details. 
\item Section \ref{sec:Symmetry-UQ} discusses a computer-generated experiment which highlights specific properties of the P-AF: 
\begin{itemize}
\item \textit{Symmetry} in the labeled patch dictionary $\mc{D}$ enables to \textit{quantify uncertainty} of patch assignments and to sample from multiple `best' solutions.
\item Patch-based regularization in connection with \textit{asymmetric} patch similarity functions may lead to labeling pattern \textit{completion}.
\end{itemize}
\item A real-world experiment in Section \ref{sec:Real-World-Exp} exemplifies how the \textit{structure} of the labeled patch dictionary $\mc{D}$ may be adapted to \textit{prior knowledge} about the application.
\end{enumerate}

\begin{figure}
\centerline{
\includegraphics[height=0.2\textwidth]{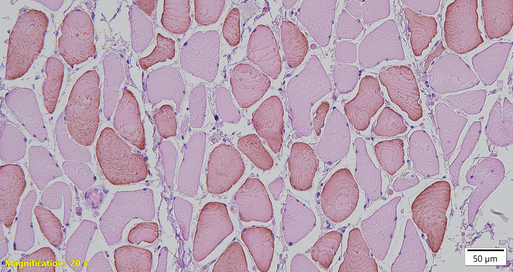}
\includegraphics[height=0.2\textwidth]{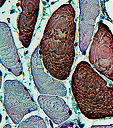}
\includegraphics[height=0.15\textwidth]{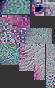}
\includegraphics[height=0.15\textwidth]{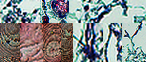}
}
\centerline{
\parbox{0.325\textwidth}{\centering (a)}\hfill
\parbox{0.2\textwidth}{\centering (b)}\hfill
\parbox{0.4\textwidth}{\centering (c)}
}
\caption{
\textbf{(a)} A light microscopy image of cross-sectional skeletal muscle structure with immunohistochemical staining (see text). \textbf{(b)} Raw input data (section) after histogram equalization (cf.~also Figure \ref{fig:MyoST-results}(e)). \textbf{(c)} Few interrogation regions with raw data from the foreground class 1 (left), class 2 and the background class (top and right), used to train a SVM for local labeling $W^{0}$ of the entire image, which defines the initial point of the patch assignment flow by \eqref{eq:def-PO}.
}
\label{eq:MyoST-input}
\end{figure}

\subsection{Implementation Details}\label{sec:Implementation-Details}

\textbf{Patch assignment flow: initialization.} Input data for the P-AF is an \textit{initial} labeling function $\ell_{\mc{V}}^{0}\colon\mc{V}\to\mc{C}$ obtained, e.g., by classifying data \textit{locally} at each vertex $i\in\mc{V}$ using any established method. Encoding these initial \textit{integral} class labels $\ell^{0}_{\mc{V}}(i) \to W^{0}_{i},\; \forall i\in\mc{V}$ by an assignment vector field $W^{0}\in\ol{\mc{S}^{n}_{c}}$ of unit vectors, the initial point of the P-AF in \eqref{eq:def-P-AF} is defined as
\begin{subequations}\label{eq:def-PO}
\begin{align}
P_{0;i,d} 
&:= \frac{\exp(\la W^{\lambda}_{[i]}, d_{[i]}\ra)}{\sum_{d'\in\mc{D}}\exp(\la W^{\lambda}_{[i]}, d'_{[i]}\ra)},\qquad i\in\mc{V},\quad d\in\mc{D}
\label{eq:def-PO-a} \\ \label{eq:def-PO-b}
W_{i}^{\lambda} &:= (1-\lambda) W^{0}_{i} + \lambda \eins_{\mc{S}_{c}},\qquad i\in\mc{V},\quad \lambda \in [0,1],
\end{align}
\end{subequations}
where $W^{\lambda}_{[i]}$ denotes the patch of assignment vectors \eqref{eq:def-PO-b} centered at $i\in\mc{V}$ and $\eins_{\mc{S}_{c}} = \frac{1}{c}\eins_{c}$ denotes the barycenter of $\mc{S}_{c}$ (uniform discrete distribution). 
The parameter $\lambda$ balances the influence of the initial labeling $W^{0}$ and the regularizing effect of the dictionary $\mc{D}$, respectively. It is the \textit{only user parameter} of the patch assignment flow.

\textbf{Label assignment and uncertainty quantification.} In all experiments, the P-AF was integrated using the geometric Euler method with the very small stepsize $h=0.02$ in order to rule out any approximation errors caused by too large stepsizes, and with a sufficiently large time interval $[0,T]$ leading always to convergence of $P(t)$. $P(T)$ then determines the regularized labeling function $\ell_{\mc{V}}$ as described in Section \ref{sec:approach1}, item (ii). In addition, in order to quantify uncertainty of patch assignments caused by \textit{multiple locally consistent} templates from the patch dictionary $\mc{D}$, the \textit{mean patch assignment function} for \textit{binary} labeling problems is defined as
\begin{equation}\label{eq:def-mean-ell}
\ol{\ell_{\mc{V}}}\colon\mc{V}\to[0,1],\qquad
\ol{\ell_{\mc{V}}}(i) 
= \frac{1}{|\mc{D}|}\sum_{j\in[i]_{\mc V}}\sum_{d\in\mc{D}} P_{j,d}(T) d_{[j]}(i).
\end{equation}
Informally, the convex combinations of all assigned dictionary patches (with weights given by $P_{j}(T)$) are pasted onto the graph at each $j\in\mc{V}$, and this weighted superposition is evaluated at each vertex $i\in\mc{V}$, to which all patches assigned in a neighborhood of $i$ contribute.

\begin{figure}[t]
\centerline{
\includegraphics[width=0.45\textwidth]{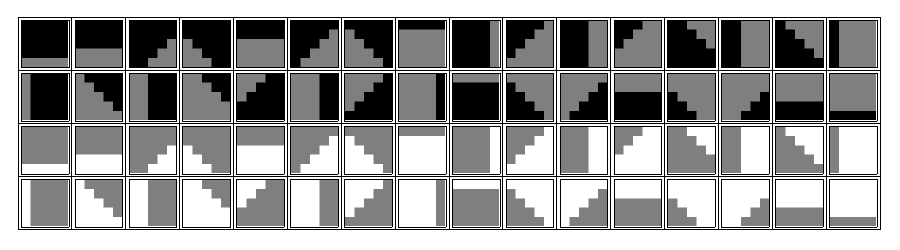}
\includegraphics[width=0.55\textwidth]{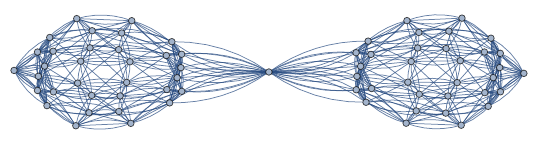}
}
\centerline{
\parbox{0.5\textwidth}{\centering (a)}
\parbox{0.5\textwidth}{\centering (b)}
}
\caption{
\textbf{(a)} A dictionary of $5\times 5$ \textit{labeled} patches with foreground labels (black and white) and background labels (gray), complemented by the three constant patches (not shown). \textbf{(b)} The patch adjacency graph is structured so as to favour transitions from each foreground class to itself or to the background, respectively, rather than direct transitions between both foreground classes. The three constant patches are represented by the center vertex (background label) and the two extreme vertices (foreground labels), respectively. This structure of the patch dictionary graph encodes the prior knowledge that spatially connected components of both foreground regions should be \textit{separated} by the background region.
}
\label{fig:dictionary-graph}
\end{figure}

\subsection{Symmetry, Uncertainty Quantification}\label{sec:Symmetry-UQ}

Figure \ref{fig:dictionary-lines-5x5} shows a small dictionary of labeled binary template patches with lines or crossing lines, together with visualizations of the adjacency structure based on the patch similarity function \eqref{eq:om-patch-binary}. The patch dictionary graph reveals that template patches are generally adjacent to multiple other template patches.

Figure \ref{fig:5x5-lines}(a) shows input data $W^{0}$ which initialize the P-AF by \eqref{eq:def-PO}. The labelings (b)-(d) determined by the P-AF, based on the dictionary graph of Figure \ref{fig:dictionary-lines-5x5}, differ by encoding no preference (default: (b)) or a slight preference for foreground structure relative to the background ((c), (d); see the caption). This encoding can be easily achieved by modifying the assignment vector field $d_{[i]}$ in \eqref{eq:def-PO-a} corresponding to the labeled template patches of the dictionary $\mc{D}$, accordingly. The P-AF returns the `energetically best' labeling as measured by the objective function \eqref{eq:def-J-P-AF}.

Panels (e)--(g) display the corresponding mean patch assignment function \eqref{eq:def-mean-ell}. Colors close to white indicate large uncertainty and plausible alternative labelings, respectively, from which one could even sample using $P(T)$. This property of the P-AF is a consequence of using \textit{labeled} patch dictionaries and the symmetry of patch dictionary graphs.


%
\begin{figure}[t]
\centerline{
\includegraphics[width=0.15\textwidth]{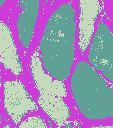}
\hspace{0.02\textwidth}
\includegraphics[width=0.15\textwidth]{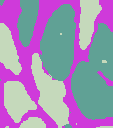}
\hspace{0.02\textwidth}
\includegraphics[width=0.15\textwidth]{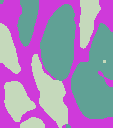}\hspace{0.02\textwidth}
\includegraphics[width=0.15\textwidth]{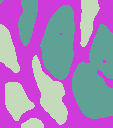}
}
\centerline{
\parbox{0.15\textwidth}{\centering (a)}\hspace{0.02\textwidth}
\parbox{0.15\textwidth}{\centering (b)}\hspace{0.02\textwidth}
\parbox{0.15\textwidth}{\centering (c)}\hspace{0.02\textwidth}
\parbox{0.15\textwidth}{\centering (d)}
}
\vspace{0.1cm}
\centerline{
\includegraphics[width=0.5\textwidth]{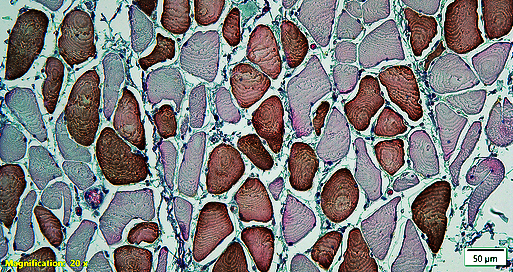}
\includegraphics[width=0.5\textwidth]{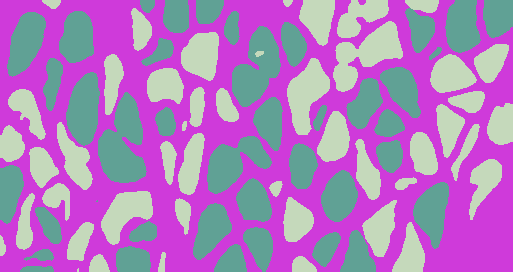}
}
\centerline{
\parbox{0.5\textwidth}{\centering (e)}
\parbox{0.5\textwidth}{\centering (f)}
}
\caption{
\textbf{(a)} Initialization of the patch-AF, obtained by pixelwise classification of the data depicted by Figure \ref{eq:MyoST-input}(b). \textbf{(b)--(d)} Results of the patch-AF with increasing regularization, which enforces transitions to the background label, in particular at locations with uncertain labeling decisions (compare (b) and (d)). \textbf{(e)} Input data resulting from preprocessing (histogram equalization) of the raw data depicted by Figure \ref{eq:MyoST-input}(a). \textbf{(f)} Labeling returned by the patch-AF. Regularization is entirely and explicitly encoded by the dictionary patch adjacency graph (Figure \ref{fig:dictionary-graph}). 
}
\label{fig:MyoST-results}
\end{figure}

\subsection{Real Data Example}\label{sec:Real-World-Exp}

Figure \ref{eq:MyoST-input}(a) shows a light microscopic image of anti-myosin immunostained cross-sectional skeletal muscle, which was obtained in a larger study of fiber type composition analysis in respiratory skeletal muscle of COVID-19 positive patients. A key preparatory step of the entire data analysis pipeline concerns the segmentation of fibers in order to measure fiber size and further morphological properties.

Figure \ref{fig:dictionary-graph} shows the chosen patch dictionary $\mc{D}$ and the corresponding dictionary graph using the similarity function \eqref{eq:om-patch-weight} (only edges with the three largest weights are shown, for better visibility). As detailed in the caption, the \textit{structure} of $\mc{D}$ has been chosen so as to enforce the expected topological structure: preference for spatial transitions between either foreground label and the background label, rather than direct transitions between foreground labels.

Figure \ref{fig:MyoST-results} shows that regularization via the interaction of labeled patch templates is effective, in particular regarding the suppression of nuisance background structure, \textit{without} any user parameter to be tuned.

\section{Conclusion}\label{sec:conclusion}

We extended the assignment flow approach towards regularized label assignments which, besides spatial regularization, additionally takes into account \textit{label interaction}. The interaction of labels is entirely encoded by a dictionary of labeled patch templates and a corresponding patch dictionary graph, which quantifies the local consistency of spatially adjacent patch template assignments. In this way, local constraints effectively constrain global labelings of image feature data, which result from geometric numerical integration of the Riemannian patch assignment gradient flow.

Our further work will study in this context the design and structure of labeled patch dictionaries, as generators of nonlocal data labelings on graphs. In particular, discrete symmetries will be examined from the general viewpoint on locally equivariant networks that are generated by geometric flows, as developed by \cite{Cassel:2025aa}.

\section*{Acknowledgments}
This work is funded by the Deutsche Forschungsgemeinschaft (DFG) under Germany's Excellence Strategy EXC-2181/2 - 390900948 (the Heidelberg STRUCTURES Excellence Cluster), and grant SCHN 457/17-2 within the priority programme SPP 2298: ``Theoretical Foundations of Deep Learning.''

\bibliographystyle{alpha} 
\bibliography{af_patches}

\end{document}